\def\year{2018}\relax
\documentclass[letterpaper]{article} 
\usepackage{aaai18}  
\usepackage{times}  
\usepackage{helvet}  
\usepackage{courier}  
\usepackage{url}  
\usepackage{graphicx}  

\usepackage{amssymb}
\usepackage[cmex10]{amsmath}
\usepackage{array}
\usepackage{algorithmic}

\usepackage{bm}
\usepackage{color}
\usepackage{dsfont}
\usepackage{amsthm}
\usepackage{setspace}
\usepackage{mathbbol}
\usepackage[ruled,norelsize]{algorithm2e}
\usepackage{bbm}
\usepackage{IEEEtrantools}
\usepackage[utf8]{inputenc}

\usepackage{pgfplots} 
\pgfplotsset{compat=newest} 
\pgfplotsset{plot coordinates/math parser=false}

\allowdisplaybreaks

\newcommand{\E}{\mathsf{E}}
\newcommand{\Tr}{\operatorname{trace}}

\newcommand{\T}{\top}
\newcommand{\mc}{\mathcal}
\newcommand{\defeq}{\triangleq}

\newcommand{\1}{\mathds{1}}

\newcommand{\btheta}{\bm{\theta}}

\newcommand{\bs}{\bm{s}}
\newcommand{\ba}{\bm{a}}

\newcommand{\Ps}{\mc{P}}

\newcommand{\N}{\mc{N}}

\renewcommand{\Re}{\mathbb{R}}
\newcommand{\St}{\mathbb{S}}
\newcommand{\Ac}{\mathbb{A}}

\newcommand{\V}{\mathbb{V}}

\newcommand{\M}{\mathbb{M}}

\newtheorem{lemma}{Lemma}

\newtheorem{corollary}{Corollary}

\begin{document}
%
\title{Diff-DAC: Distributed Actor-Critic for Average Multitask \\
Deep Reinforcement Learning}
\author{
	Sergio Valcarcel Macua$^\star$
	\and
	Aleksi Tukiainen$^\star$
\AND
	Daniel Garc\'ia-Ocaña Hern\'andez$^\dagger$
	\and
	David Baldazo$^\dagger$
\AND
	Enrique Munoz de Cote$^\star$
	\and
	Santiago Zazo$^\dagger$\\
	$^\star$PROWLER.io\\
	$^\dagger$Information Processing and Telecommunications Center - 
	Universidad Polit\'ecnica de Madrid
}
\maketitle

\begin{abstract}
We propose a fully distributed actor-critic algorithm approximated by deep neural networks, named \textit{Diff-DAC}, with application to single-task and to average multitask reinforcement learning (MRL).
Tasks share state-action sets and reward function, but have different state transition distribution.
Each agent has access to data from its local task only, but it aims to learn a policy that performs well on average for the whole set of tasks. 
During the learning process, agents communicate their value-policy parameters to their neighbors, diffusing the information across the network, so that they converge to a common policy, with no need for a central node.
The method is scalable, since the computational and communication costs per agent grow with its number of neighbors.
We derive Diff-DAC's from duality theory and provide novel insights into the standard actor-critic framework, showing that it is actually an instance of the dual ascent method that approximates the solution of a linear program. 
Experiments suggest that Diff-DAC can outperform the single previous distributed MRL approach (i.e., Dist-MTLPS) and even the centralized architecture.\footnote{Presented at Adaptive Learning Agents workshop (ALA2018), July 14th, 2018, Stockholm, Sweden.}
\end{abstract}

%
\section{Introduction}
%
%
%

\noindent Within a decade, billions of interconnected devices will be processing and exchanging data throughout the global economy \cite{gartner2015}. 
Centralised reinforcement learning (RL) architectures where all devices interact with a central station might be unfeasible. 
Fully distributed RL algorithms, where agents communicate only with neighbors and without central control, offer a solution to this problem, 
since the communication cost per agent scales linearly with its number of neighbors. 
In this distributed approach, each agent learns by interacting with its own environment, but is able to cooperate and benefit from the learning process of the whole network.
When all agents' environments are equal, they learn to perform a single task; 
when environments are different but related, 
they learn to generalize across all tasks \cite{taylor2009transfer}.
The latter is known as the multitask reinforcement learning (MRL) problem.
We propose an algorithm named \textit{Diffusion-based Distributed Actor-Critic} (Diff-DAC) for both single-task and MRL problems.

Most previous MRL approaches assumed access to data from all tasks 
\cite{bou-ammar2014online,parisotto2015actor,teh2017distral}.
But if the number of tasks is large and their data are geographically distributed,
the communication cost of transmitting the data to a central station might be prohibitive.

The idea of making scalable MRL with distributed optimization was first proposed by \citeauthor{el2017scalable} \shortcite{el2017scalable} with the Dist-MTLPS method, 
which extended a distributed implementation of ADMM due to \citeauthor{wei2012distributed} \shortcite{wei2012distributed}.
Our work improves over Dist-MTLPS in a number of ways: 
	\textit{i)} Dist-MTLPS relies on linear function approximation, which requires finding salient features,
	and it only considers policies in the natural exponential family of distributions.
	Diff-DAC, on the other hand, uses deep learning architectures to avoid costly feature engineering, and is able to learn more expressive policies.
	\textit{ii)} The distributed ADMM updates of the agents are done in sequential order,
	requiring finding a cyclic path that visits all agents,
	which is generally an NP-hard problem \cite{karp1972reducibility}.
	Diff-DAC uses a diffusion strategy \cite{sayed2014adaptation}, 
	where each agent interacts with its neighbors, with no ordering, 
	and possibly asynchronously \cite{zhao2015asynchronousI}.
	\textit{iii)} Sequential strategies are sensitive to agent or link failures, 
	since they stop the information flow;
	while diffusion strategies are robust since the agents can still operate even if parts of the network become isolated.

As far as we know, all other previous works on distributed RL only considered tabular or linear function approximations (e.g., \cite{Kar2012,valcarcel2013distributed,tutunov2016exact}),
and do not apply immediately to expressive nonlinear approximations.
In particular, reference \cite{Kar2012} added a consensus rule to tabular Q-learning; a principled nonlinear extension raises questions like whether we should we add consensus to the target network updates, and would be an alternative contribution to our actor-critic approach. 
The Dist-GTD method due to \cite{valcarcel2013distributed} is for policy evaluation with linear approximation, and tts extension to control and nonlinear approximations isn't trivial even for the single-agent GTD.
Finally, reference \cite{tutunov2016exact} proposed a second order method, implying the inversion of the Hessian at every agent,
which might be problematic for neural networks with hundreds of neurons.
Other related works suffer from similar drawbacks.

\textbf{Contributions.} 
(1) We propose a fully distributed actor-critic deep reinforcement learning algorithm named Diff-DAC for the single and average multitask problem that scales gracefully to large number of tasks.
(2) We re-derive the actor-critic framework from duality theory and show that it is an instance of \textit{dual-ascent} to approximate the saddle-point of the Lagrangian of a linear program (LP). 
This derivation formalizes previous intuitions \cite{pfau2016connecting} and provides novel insights,
like a policy gradient that includes the advantage function explicitly, rather than as a variance reduction technique.
(3) Experimental results suggest that Diff-DAC outperforms Dist-MTLPS,
and that it is more stable and achieves better local optima than the centralized approach, without replay memory or target networks.

%
\section{Problem Formulation}
\label{sec:problem-formulation}
%

In this section, we formalize tasks as Markov decision processes (MDPs),
define a family of tasks and introduce the multitask optimization problem. 

Consider a parametric family of MDPs defined over finite\footnote{The proposed Diff-DAC algorithm uses function approximation so that is able to work in continuous state-action sets as well.} 
state-action sets, $\St$ and $\Ac$.
Each MDP of the family has different state transition distribution, $\Ps_{\theta} (s'|s,a)$,
that depend on some parameter $\theta \in \Theta$,
where $\Theta$ is a measurable compact set.
The task family is given as a probability distribution over the parameter set, $f$, 
so that the parameter is a random variable\footnote{We use boldface font to denote random variables and regular font to denote instances or deterministic variables.}:
$\btheta = \theta \sim f$.
All MDPs in the faminly share the
reward function, $r(s,a)$, and the distribution over the initial state, $\mu(s)$, $\forall s,s' \in \St$, $a \in \Ac$,

Let $\pi : \St \times \Ac \mapsto [0,1]$ be a stationary policy,
such that $\pi(a|s)$ denotes the probability of taking action $a$ at state $s$.
Let $v: \Pi \times \St \mapsto \Re$ denote the value function,
such that $v_{\theta}^{\pi}(s)$ is the value at state $s $ when following policy $\pi$:
\begin{IEEEeqnarray}{rCl}
	v^{\pi}_{\theta}(s) 
\defeq  
	\E_{\pi, \mc{P}_{\theta}}
		\left[
			\sum_{t=0}^{\infty}
				\gamma^{t} 
				r ( \bm{s}_t, \bm{a}_t )
		\: 
		\big| 
		\:
			\bm{s}_0 = s
		\right]
,
\label{eq:task-value-function}
\end{IEEEeqnarray}
where $\E_{\pi, \mc{P}_{\theta}} \left[  \cdot \right]$ is the expected value when $\ba_t \sim \pi(\cdot|s_t)$ 
and $\bs_{t+1} \sim \mc{P}_{\theta} (\cdot | s_t, a_t)$;
and
$
	0 < \gamma <	1
$
is the discount factor.
Introduce the vector of values:
$
	v^{\pi}_{\theta} 
\defeq  
	\left(
		v^{\pi}_{\theta}(s)
	\right)_{s \in \St}
\in 
	\Re^{|\St|}
$.

Suppose we have observed $N$ tasks that correspond to parameters $\left \{ \theta_k \right\}_{k=1}^N$.
%
%
%
Our goal is to learn a stationary policy that maximizes the \textit{global} average value:
\begin{IEEEeqnarray}{rCl}
\underset{  \pi }{\rm maximize} 	
\quad	
		\mu^\T 
		\left(
			\sum_{k=1}^N 
				v_{\theta_k}^{\pi}
		\right)
,
\label{eq:empirical-risk}
\end{IEEEeqnarray}

We assume  bounded rewards:
\begin{IEEEeqnarray}{rCl}
	| r(s,a) |
&
\le 
&
	R_{\max} < \infty
,\;\:
	\forall (s,a) \in \St \times \Ac 
,\quad
\label{eq:bounded-reward}
\end{IEEEeqnarray}
for some scalar $R_{\max}$.
Under this assumption we can easily ensure existence of solution to \eqref{eq:empirical-risk}.
%
%

When all task parameters $\left \{ \theta_k \right\}_{k=1}^N$ are equal, \eqref{eq:empirical-risk} is the single-task RL problem;
when they differ, \eqref{eq:empirical-risk} becomes an MRL problem where we aim to learn a single policy that performs optimally in average for the whole set of tasks.

%
\section{Networked Multiagent Setting}
\label{sec:multiagent-approach}
%

In this section we introduce the networked multiagent setting that learns in a fully distributed manner.

\begin{figure}
\centering \includegraphics[width=0.95\columnwidth]{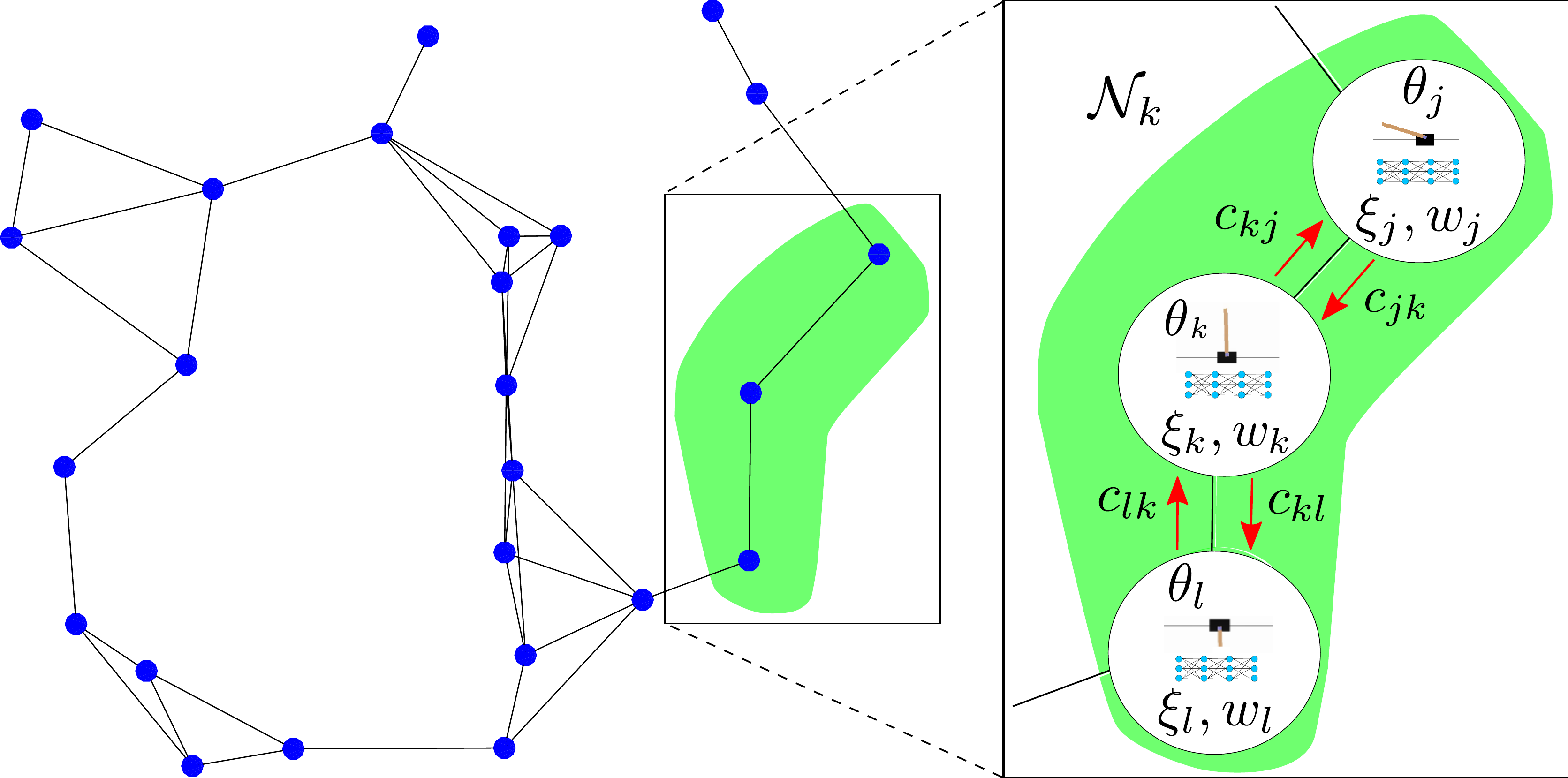}
\caption{Example of network and detailed neighborhood.
Blue nodes represent agents, and edges represent their connectivity.
This network consists of $N=25$ agents, with average neighborhood size:
$
	|\overline{\N}| = \frac{1}{N}\sum_{k=1}^N |\N_k| 
= 
	4.2
$.
On the right, the figure zooms over neighborhood $\N_k$ (green area), 
where each agent $k$ runs its own task instance of the swing-up cart-pole task (with diffent pole lenght and mass, defined by parameter $\theta_k$).
As explained later, agent $k$ transmits its neural network weights, $\xi_{k,i}$ and $w_{k,i}$, to its neighbors $j$ and $l$;
and it receives their weights $\xi_{j,i}, w_{j,i}$ and $\xi_{l,i}, w_{l,i}$, 
and combines them with coefficients $c_{jk}$ and $c_{lk}$, respectively.}
\label{fig:network}
\end{figure}

We have a network of $N$ agents, which is expressed as a graph, $\N$.
Each node, denoted $k=1,\ldots,N$, corresponds to an agent that learns from data coming from its own task\footnote{For simplicity, 
we assume that each agent is allocated with one task,
similar to \cite{el2017scalable}. 
The extension to multiple tasks per agent is trivial.},
with parameter $\theta_k \sim f$. 
The edges in the graph represent communication links.
We assume that the graph is \textit{connected} (i.e., there is at least one path between every pair of agents).

The graph can be represented by a non-negative matrix of size $N \times N$, denoted
%
$
	C 
\defeq
	\left( c_{lk} \right)_{l,k=1}^N
$,
such that the element $c_{lk} \ge 0$ represents the weight given by agent $k$ to information coming from $l$.
Each agent $k$ is only allowed to communicate within its own neighborhood, $\N_k$, 
which is defined as the set of agents to which it is directly connected, including $k$ itself:
%
$
	\N_k 
\defeq 
	\left\{
		l \in \{1,\ldots,N \}: c_{lk} > 0
	\right\}
$.
%
%
%
%
In order to ensure that the information flows through the network, 
we require the following standard conditions on the connectivity matrix $C$,
which make it doubly-stochastic and primitive \cite{sayed2014adaptation,valcarcel2013distributed}:
\begin{IEEEeqnarray}{rCl}
	C^\T \1 = \1
,\;
	C \1 = \1
,\;
\text{ and }\;
	c_{lk} 
&
\ge
&	
	0
,\; k,l = 1,\ldots,N
,\quad
\label{eq:non-negative-coefficients}
\\
	\Tr \:[C]
&
>
& 
	0
,\quad
\label{eq:aperiodic}
\end{IEEEeqnarray}
where $\1$ is a vector of ones.
Although conditions \eqref{eq:non-negative-coefficients}--\eqref{eq:aperiodic} seem restrictive, 
it turns out that there are procedures for every agent $k$ to find the weights $\left\{ c_{lk} \right\}_{l\in\N_k}$ in a fully distributed manner,
such that $C$ satisfies the required conditions. 
One of such procedures is the Hastings rule \cite{zhao_performance_2012}, \cite[p.492]{sayed2014adaptation}.
%

%
\section{Multitask Actor-Critic from Duality Theory}
\label{sec:actor-critic}
%

In this section, we reformulate MRL as a linear program (LP),
and show that by applying dual ascent to the Lagrangian,
we get a tabular model-based actor-critic method that solves \eqref{eq:empirical-risk}.

Introduce the average dynamics:
\begin{IEEEeqnarray}{rCl}
	\overline{\Ps}(s'|s,a)
&
\defeq
&
	\frac{1}{N}
	\sum_{k=1}^N
		\Ps_{\theta_k}(s'|s,a)
.
\end{IEEEeqnarray}
The following lemma is key in our derivations, 
since it allow us to consider the multitask problem as a single MDP, 
with state-action transitions given $\overline{\Ps}$.
%
%
\begin{lemma}
\label{theorem:cooperative-transition-matrix-stochastic}
$\overline{\Ps}$ is a row-stochastic matrix.
\end{lemma}
%
%
%
%
\begin{proof}
Stochastic matrices lie in a compact convex set \cite[p.527]{horn1990matrix},
so that their convex combination lies in the same set \cite[p.24]{boyd2004convex}.
\end{proof}
%
%
%
Thus, we can use standard (i.e., single-task) optimal control results \cite[pp. 143--151]{Puterman2005} for our setting.
Let $\V$ denote the set of bounded real functions on $\St$, with componentwise partial order and norm $\|v\| \defeq \sup_{s\in\St}|v(s)|$.
%
%
\begin{corollary}
\label{corollary:cooperative-value-satisfies-bellman}
For any stationary policy, $\pi$,
the Bellman equation for the new MDP defined by $\overline{\Ps}$ is given by:
\begin{IEEEeqnarray}{rCl}
	v(s)
&
=
&
	\sum_{a\in\Ac}
	\pi(a|s)
	\left(
		r(s,a)
		+
		\gamma
		\sum_{s'\in\St}
			\overline{\Ps}(s'|s,a)
			v(s')			
	\right)
,
\quad\:
\label{eq:multitask-bellman-equation}
\end{IEEEeqnarray}
\end{corollary}
%
%
%
Introduce the multitask Bellman operator $\overline{T}:\V \mapsto \V$:
\begin{IEEEeqnarray}{rCl}
	\left( \overline{T} v \right)(s)
\defeq
	\max_{a \in \Ac}
	\left[
		r(s,a)
	    +
		\gamma
			\sum_{s'\in \St}\overline{\Ps}(s'|s,a)
			v(s')
	\right]
.
\label{eq:Bellman-operator-vector}
\end{IEEEeqnarray}
Similar to single task optimal control theory \cite[Sec. 9.1]{Puterman2005},
we can reformulate \eqref{eq:empirical-risk} as an LP.
\begin{IEEEeqnarray}{rCl}
\begin{aligned}
\underset{ v \in \Re^{|\St|} }{\rm minimize} 	
	&\;\;
		\mu^\T v
\\
{\rm s.t.}
	\quad\;\; &
		v(s)
	\ge 
		r(s,a)
		+
		\gamma
		\sum_{s' \in \St}
			\overline{\Ps} (s'|s,a) v(s')
\qquad\\
	&
	\forall (s,a) \in \St \times \Ac
,
\end{aligned}
\label{eq:cooperative-linear-program}
\end{IEEEeqnarray}
%


Since problem \eqref{eq:cooperative-linear-program} satisfies Slater condition, 
strong-duality holds \cite[Sec. 5.2.3]{boyd2004convex} and the primal and dual optimal values are attained and equal.
The Lagrangian of \eqref{eq:cooperative-linear-program} is given by:
\begin{IEEEeqnarray}{rCl}
	L(v,d)
&
=
&
	\mu^\T v
	+
	\sum_{(s,a)\in\St\times\Ac}
	d(s,a)
	\bigg(
    	r(s,a)
\notag\\
&&
		+
\:
        \gamma
		\sum_{s' \in \St}
        	\overline{\Ps} (s'|s,a) v(s')
        	-
        	v(s)
	\bigg)
,
\label{eq:cooperative-lagrangian}
\end{IEEEeqnarray}
where the dual variable 
$
	d 
\defeq
	\left(
		d(s,a)
	\right)_{(s,a)\in\St\times\Ac}
\in
	\Re^{|\St||\Ac|}
$ 
is a nonnegative vector of length $|\St||\Ac|$.
Let $v^\star$ denote the optimal primal variable.
Let $d^{\star}$ denote an optimal dual variable, which might not be unique.
The saddle point condition of \eqref{eq:cooperative-lagrangian} is given by:
\begin{IEEEeqnarray}{rCl}
	\min_{v} 
	\max_{d}
	L(v,d)
&
=
&
	L(v^{\star}, d^{\star})
=
	\max_{d}
	\min_{v} 
	L(v,d)
.
\quad
\label{eq:saddle-point}
\end{IEEEeqnarray}
%
%
%
There are multiple approaches to find a saddle point.
We focus on the \textit{dual-ascent} scheme
\cite{arrow1958studies},
which consists in alternating between:
\textit{i)} Finding a primal solution, 
given the dual variable;
and 
\textit{ii)} ascending in the direction of $\nabla_d L(v,d)$, 
given the primal variable.

First, we show how to update the \textit{primal} variable given $d$:
\begin{IEEEeqnarray}{rCl}
	v
\leftarrow
	\arg\min_{v\in\Re^{|\St|}} L(v,d)
.
\end{IEEEeqnarray}
Since problem \eqref{eq:cooperative-linear-program} is linear, if we derive the KKT conditions, we can see that the first-order condition holds $\forall v \in \Re^{|\St|}$.
Thus, the only condition that depends on $v$ is \textit{complementary slackness}:
\begin{IEEEeqnarray}{rCl}
	\sum_{(s,a)\in \St \times \Ac}
&&
		d(s,a)
		\left(
			r(s,a)
			+
		    \gamma
		    \sum_{s' \in \St}
				\overline{\Ps}(s'|s,a)
				v(s')
			-
			v(s)
		\right)
\notag\\
&&=
\:
	0
.
\label{eq:cooperative-original-complementary-slackness}
\end{IEEEeqnarray}
Similar to the standard single-task problem \cite[Sec. 6.9]{Puterman2005}, 
it can be shown that our dual variable is the discounted steady-state state-action visitation measure,
so that we can obtain the policy from $d$:
\begin{IEEEeqnarray}{rCl}
	\pi(a|s)
=
	\frac{d(a,s)}{\rho(s)}
,
\label{eq:policy-from-dual-variable}
\;\;\text{where}\;\;
	\rho(s) 
\defeq 
	\sum_{a'\in \Ac} d(a',s)
.
\quad
\end{IEEEeqnarray}
Hence, the Bellman equation \eqref{eq:multitask-bellman-equation},
typically used to derive the \textit{critic} in actor-critic methods,
is sufficient to guarantee \eqref{eq:cooperative-original-complementary-slackness}.

Second, for the \textit{dual} variable, we simply perform gradient ascent in the Lagrangian, yielding a recursion of the form:
\begin{IEEEeqnarray}{rCl}
	d
&
\gets
&
	\left[
		d
		+
		\alpha
		\nabla_d L(v,d)
	\right]^+
,
\label{eq:dual-gradient-ascent-update-of-cooperative-problem}
\end{IEEEeqnarray}
where $\alpha$ is the step-size,
$[\cdot]^+$ denotes projection on the nonnegative quadrant,
and the $\nabla_d $ denotes gradient w.r.t. dual variable $d$:
\begin{IEEEeqnarray}{rCl}
	\nabla_d L(v,d) 
&
= 
&
	\left(
		\frac{\partial L(v,d)}{\partial d(s,a)}
	\right)_{ (s,a) \in \St \times \Ac }
.
\label{eq:gradient-Lagrangian-wrt-dual-variable}
\end{IEEEeqnarray}
Interestingly, 
note that the partial derivatives of the Lagrangian in \eqref{eq:gradient-Lagrangian-wrt-dual-variable} are indeed 
the so named \textit{advantage function} 
extended to our multitask problem:
\begin{IEEEeqnarray}{rCl}
	A(s,a)
&
\defeq
&
	r(s,a)
	+
	\gamma
	\sum_{s' \in \St}
        	\overline{\Ps} (s'|s,a) v(s')
        	-
        	v(s)
\notag\\
&
=
&
	\frac{\partial L(v,d)}{\partial d(s,a)}
.
\quad
\label{eq:multitask-advantage}
\end{IEEEeqnarray}
If we learn $d^{\star}$, we can use \eqref{eq:policy-from-dual-variable} to obtain $\pi^{\star}$,
so that recursion \eqref{eq:dual-gradient-ascent-update-of-cooperative-problem} can be seen as an \textit{actor} update.
Thus, \eqref{eq:multitask-bellman-equation} and \eqref{eq:dual-gradient-ascent-update-of-cooperative-problem} define a novel \textit{tabular model-based} \textit{actor-critic} method.
In the following, we extend this approach to a model-free distributed actor-critic method with neural network approximations.

%

%
\section{Distributed Deep Actor-Critic}
\label{sec:distributed-deep-actor-critic}
%
%

In order to use \textit{diffusion} strategies \cite{sayed2014adaptation} to derive a distributed optimization method,
we have to express the global objective as a convex combination of each agent's local objective. 
Thus every agent can optimize its objective from local data;
and by communicating with their neighbors,
all agents converge to a common solution that optimizes the global objective.
We do this for both critic and actor.

%
\subsection{Distributed policy evaluation: Critic}
%
%

When computing the critic for large (or continuous) state-action sets, 
it is common to approximate the value function with some parametric function $v_{\xi}(s) \approx v(s)$, 
where $\xi \in \Re^{M_v}$ denotes the parameter vector of length $M_v$.
%
We choose neural networks with multiple hidden layers (i.e., deep learning) as parametric approximators.
Hence, we can learn the network weights, $\xi$, 
by transforming \eqref{eq:multitask-bellman-equation} into a nonlinear regression problem:
\begin{IEEEeqnarray}{rCl}
\underset{  \xi \in \Re^{M_v} }{\rm minimize} 	
&&
\quad
	J(\xi)
\defeq
		\E
	    \left[
	      \left (
	          v_{\xi} \left( \bm{s}_t \right)
	          -
	          \overline{\bm{y}}_t
	      \right )^2
	    \right]
,
\label{eq:distributed-expected-bellman-error}
\end{IEEEeqnarray}
where the target values are given by:
\begin{IEEEeqnarray}{rCl}
	\overline{\bm{y}}_t
&
\defeq
&
	r(\bm{s}_t, \bm{a}_t) 
	+ 
	\gamma
	\sum_{s' \in \St}
    	\overline{\Ps} (s'|\bm{s}_t,\bm{a}_t) v_{\xi}(s')
.
\end{IEEEeqnarray}

In order to derive a diffusion-based distributed critic, 
we have to reformulate the problem as minimizing the convex combination of costs that depend only on a single task each.
The cost for each individual task takes the form:
\begin{IEEEeqnarray}{rCl}
	\widetilde{J}_k(\xi)
&
\defeq
&
	\E
	\left[
		\left (
			v_{\xi} \left( \bm{s}_{t} \right)
			-
			\bm{y}_{k,t}
      \right )^2
	\right]
,\;
	k=1,\ldots,N
,
\label{eq:individual-critic-cost}
\end{IEEEeqnarray}
where $\bm{y}_{k,t}$ is the target from task $k$ at time $t$, given by
\begin{IEEEeqnarray}{rCl}
	\bm{y}_{k,t}
&
=
&
	r_{\theta_k}(\bm{s}_{t}, \bm{a}_{t})
	+ 
    \gamma 
	\sum_{s' \in \St}
    	\Ps_{\theta_k} (s'|\bm{s}_{t},\bm{a}_{t}) v_{\xi}(s')
,\notag
\end{IEEEeqnarray}
such that 
$
	 \overline{\bm{y}}_t
=
	1/N
	\sum_{k=1}^N
	\bm{y}_{k,t}
$.
Now, in order to obtain a cost that is a combination of the individual costs,
we can use  Jensen's inequality to upper bound $J(\xi)$ by another function, 
$
	\widetilde{J}(\xi)
$, and use this upper bound as surrogate cost:
\begin{IEEEeqnarray}{rCl}
	\widetilde{J}(\xi)
&
\defeq
&
	\frac{1}{N}
	\sum_{k=1}^N
	\widetilde{J}_k(\xi)
=
	\frac{1}{N}
	\sum_{k=1}^N
	\E
    \left[
		\left (
			v_{\xi} \left( \bm{s}_{t} \right)
			-
			\bm{y}_{k,t}
      \right )^2
    \right]
\notag\\
&
\ge 
&
	\E
    \left[
		\left (
			\frac{1}{N}
			\sum_{k=1}^N
				\left(
					v_{\xi} \left( \bm{s}_{t} \right)
					-
					\bm{y}_{k,t}
				\right)
      \right )^2
    \right]
=
	J(\xi)
.
\qquad
\end{IEEEeqnarray}
%
%
%

Now, we can apply \textit{diffusion} stochastic-gradient-descent (SGD) strategies \cite{sayed2014adaptation},
which consist of two steps: \textit{adaptation} and \textit{combination}.
During the \textit{adaptation} step, each agent performs SGD on its individual cost, $\widetilde{J}_k$, 
to obtain some intermediate-parameter update.
Then, each agent \textit{combines} the intermediate-parameters from its neighbors.
These two steps are described by the following updates, which run in parallel for all agents $k=1,\ldots,N$:
\begin{IEEEeqnarray}{rCl}
\IEEEnosubnumber
	\widehat{\xi}_{k,i+1}
&
=
&
	\xi_{k,i}
	-
	\alpha_{i+1}
	\widehat{\nabla}_\xi \widetilde{J}_k (\xi_{k,i})
,
\IEEEyessubnumber \label{eq:value-adaptation}
\\
	\xi_{k,i+1}
&
=
&
	\sum_{l \in \N_k}
		c_{lk}
		\widehat{\xi}_{l,i+1}
,
\IEEEyessubnumber \label{eq:value-combination}
\end{IEEEeqnarray}
where 
$i$ is the iteration index;
$\alpha_i$ is the step-size; 
and 
$\widehat{\nabla}_\xi \widetilde{J}_k (\xi_{k,i})$ is the stochastic gradient evaluated at $\xi_{k,i}$,
estimated from samples $\left\{ \left( s_{k,t}, a_{k,t}, r_{k,t+1}, s_{k,t+1} \right) \right \}_{t=0}^{T_{k,i}}$ of the $i$-th episode, of length $T_{k,i}$, gathered by the $k$-th agent.
We use Monte Carlo estimates for the target
$
	y_{k,t}
=
	\sum_{j=t}^{T_{k,i}}
		\gamma^{j-t}
		r_{k,j+1}
$
(for simplicity),
where $r_{k,j+1} \defeq r_{\theta_k}(s_{k,j},a_{k,j})$ is a shorthand.
Then, the stochastic gradient is given by:
\begin{IEEEeqnarray}{rCl}
	\widehat{\nabla}_\xi \widetilde{J}_k (\xi_{k,i})
&
=
&
	\frac{1}{T_{k,i}}
	\sum_{t=0}^{T_{k,i}}
		\nabla_\xi v_{\xi_{k,i}} \left( s_{k,t} \right)
		\left(
			v_{\xi_{k,i}} \left( s_{k,t} \right)
			-
			y_{k,t}
		\right)
.
\notag\\
\label{eq:critic-stochastic-gradient}
\end{IEEEeqnarray}

We remark that each agent learns from its current episode, without replay buffer, 
similar to A3C \cite{mnih2016asynchronous}, but in a fully distributed fashion, 
as opposed as having multiple threads updating the same neural network at a single location.

%
\subsection{Distributed policy gradient: Actor}
%
%

For large state-action sets, it is convenient to approximate the policy with a parametric function.
Again, we consider expressive deep neural networks for the policy.
From \eqref{eq:policy-from-dual-variable},
we can rewrite the Lagrangian as:
\begin{IEEEeqnarray}{rCl}
	L(v,\pi,\rho)
&
=
&
	\mu^\T v
	+
	\sum_{(s,a)\in\St\times\Ac}
	\pi(a|s)
	\rho(s)
   	A(s,a)
.
\quad
\label{eq:cooperative-lagrangian-in-policy}
\end{IEEEeqnarray}

Let $\pi_w \approx \pi$ denote the parametric approximation of the actual policy,
where $w \in \Re^{M_\pi}$ is the parameter vector of length $M_\pi$.
Replacing $\pi$ with $\pi_w$ in \eqref{eq:cooperative-lagrangian-in-policy},
we obtain an approximate Lagrangian,
$
	\widetilde{L}(v,w,\rho)
\approx
	L(v,\pi,\rho)
$, of the form:
\begin{IEEEeqnarray}{rCl}
	\widetilde{L}(v,w,\rho)
&
=
&
	\mu^\T v
	+
	\sum_{(s,a)\in\St\times\Ac}
	\pi_w(a|s)
	\rho(s)
	A(s,a)
.
\quad\;\;
\label{eq:approximate-lagrangian}
\end{IEEEeqnarray}
Thus, in order to approximate the saddle point condition \eqref{eq:saddle-point},
we can move in the ascent direction of the gradient of \eqref{eq:approximate-lagrangian} w.r.t. the policy parameter, which is given by:
\begin{IEEEeqnarray}{rCl}
	\nabla_w \widetilde{L}
&&
	(v,w,\rho)
=
	\nabla_{\pi_w} \widetilde{L} (v,w,\rho)
	\nabla_w \pi_w
\notag\\
&&=\:
	\left(
		\nabla_w \pi_w(a_1|s_1)
		,\ldots,		
		\nabla_w \pi_w(a_{|\Ac|}|s_{|\St|})
	\right)^{\T}
\notag\\
&&\quad\;
	\left(
		\frac{\partial \widetilde{L} (v,w,\rho)}{\partial \pi_w(a|s)}
	\right)_{(s,a)\in\St\times\Ac}
\notag\\
&&=\:
	\sum_{(s,a)\in \St \times \Ac}
		\nabla_w \pi(a|s)
		\frac{\partial L(v,w,\rho)}{\partial \pi_w(a|s)}
\notag\\
&&=\:
	\sum_{s \in \St}
		\rho(s)
		\sum_{a \in \Ac}
			\nabla_{w} \pi_{w}(a|s)
			A(s,a)
\notag\\
&&=\:
	\sum_{s \in \St}
		\rho(s)
		\sum_{a \in \Ac}
			\pi_{w}(a|s)
			\nabla_{w} \log \pi_{w}(a|s)
			A(s,a)
,
\qquad
\label{eq:gradient-approximate-lagrangian}
\end{IEEEeqnarray}
where we used:
$
	\nabla_{w} \pi_{w}(a|s)
=
	\pi_{w}(a|s) \nabla_{w} \log \pi_{w}(a|s)
$.

Interestingly, \eqref{eq:gradient-approximate-lagrangian} is similar to previous \textit{policy gradient} theorems \cite{Sutton1999policygradient}, 
with the important difference that it yields the advantage function explicitly;
while previous works motivated the \textit{baseline} mainly as a variance reduction technique 
\cite{williams1992simple,BhatnagarNaturalActorCritic2009}.

In order to derive a fully \textit{distributed} actor,
let us write the multitask advantage function \eqref{eq:multitask-advantage} as the convex combination of advantage functions for the individual tasks:
\begin{IEEEeqnarray}{rCl}
	A(s,a)
&
=
&
	\frac{1}{N}
	\sum_{k=1}^N
		A_k(s,a)
,
\\
	A_k(s,a)
&
\defeq
&
	r_{\theta_k}(s,a)
	+
	\gamma
	\sum_{s' \in \St}
        	\Ps_{\theta_k} (s'|s,a) v(s')
        	-
        	v(s)
.\quad\;\;\:
\label{eq:individual-advantage}
\end{IEEEeqnarray}
Hence, we write the approximate Lagrangian for each task:
\begin{IEEEeqnarray}{rCl}
	\widetilde{L}_k(v,w,\rho)
&
\defeq
&
	\mu_{\theta_k}^\T v
	+
	\sum_{(s,a)\in\St\times\Ac}
	\pi_w(a|s)
	\rho(s)
	A_k(s,a)
,
\quad\;\;
\label{eq:individual-approximate-lagrangian}
\end{IEEEeqnarray}
such that 
%
$
	\widetilde{L}(v,w,\rho)
=
	\frac{1}{N}
	\sum_{k=1}^N
		\widetilde{L}_k(v,w,\rho)
$ 
.
%

Similar to the critic, once we have expressed the multitask approximate Lagrangian as the convex combination of the approximate Lagrangian of each individual task,
we can apply diffusion SGD to perform the actor update,
with smaller step-size, $\beta_{i+1} \le \alpha_{i+1}$, to approximate convergence of the critic at every actor update:
%
\begin{IEEEeqnarray}{rCl}
\IEEEnosubnumber
	\widehat{w}_{k,i+1}
&
=
&
	w_{k,i}
	+
	\beta_{i+1}
	\widehat{\nabla}_w \widetilde{L}_k(v_{\xi_{k,i}},w_{k,i},\rho)
,
\IEEEyessubnumber \label{eq:policy-adaptation}
\\
	w_{k,i+1}
&
=
&
	\sum_{l \in \N_k}
		c_{lk}
		\widehat{w}_{l,i+1}
,
\IEEEyessubnumber \label{eq:policy-combination}
\end{IEEEeqnarray}
where each agent estimates its stochastic gradient as:
\begin{IEEEeqnarray}{rCl}
	\widehat{\nabla}_{w} \widetilde{L}_k (v_{\xi_{k,i}},w,\rho)
&=&
	\frac{1}{T_{k,i}}
	\sum_{t=0}^{T_{k,i}}
	\nabla_{w} \log \pi_{w}(a_{k,t}|s_{k,t})
 	\widehat{A}_{k,t}
,
\notag\\
\label{eq:stochastic-gradient-individual-approximate-lagrangian}
\end{IEEEeqnarray}
and $\widehat{A}_{k,t}$ can be any approximation of the advantage function \cite{schulman2017proximal}.
We use the simple estimate:
\begin{IEEEeqnarray}{rCl}
	\widehat{A}_{k,t}
&
=
&
	\sum_{j=t}^{T_{k,i}}
		\gamma^{j-t}
		r_{k,j+1}
		-
		v_{\xi_{k,i}}(s_{k,t})
.
\label{eq:approximate-advantage}
\end{IEEEeqnarray}
%
%
%
Two remarks: \textit{i)} In order to simplify the implementation, we set the target $y_{k,t}$ to be the empirical return,
so that the stochastic gradient of the critic in \eqref{eq:critic-stochastic-gradient} is the negative advantage estimate:
$ \widehat{\nabla}_\xi \widetilde{J}_k (\xi_{k,i}) = - \widehat{A}_{k,t}$.
\textit{ii)} Note that replacing ${\xi_{k,i}}$ with ${\xi_{k,i+1}}$ in \eqref{eq:policy-adaptation}--\eqref{eq:policy-combination}
and \eqref{eq:stochastic-gradient-individual-approximate-lagrangian}--\eqref{eq:approximate-advantage}, 
in a Gauss-Seidel fashion, 
might lead to faster convergence.

A detailed description of Diff-DAC is given in Algorithm \ref{alg:diff-dac}.
\begin{algorithm}[!h]
\caption{Diff-DAC. This algorithm runs in parallel at every agent $k=1,\ldots,N$.}
\label{alg:diff-dac}
\textbf{Input:} Maximum number of episodes $E$,
	maximum number of steps per episode $T$,
	learning rate sequences ($\alpha_i, \beta_i$).\\
\begin{algorithmic}[1]
	\STATE Initialize critic, $v_{\xi_{k,0}}$, and actor, $\pi_{w_{k,0}}$, networks, $\forall k \in \N$.
    \STATE Initialize episode counter, $i = 0$.
    \STATE \textbf{while} $i < E$\textbf{:}
    \STATE \hspace{1em} Initialize empty trajectory, $\M_k = \{\}$.
    \STATE \hspace{1em} Initialize step counter: $t = 0$.
    \STATE \hspace{1em} Observe $s_{k,0}$.
    \STATE \hspace{1em} \textbf{while} $t < T$ and not terminal state\textbf{:}
    \STATE \hspace{2em} Select action $a_{k,t} \sim \pi_{w_{k,t}}(\cdot | s_{k,t})$.
    \STATE \hspace{2em} Execute $a_{k,t}$ and observe $r_{k,t+1}$ and $s_{k,t+1}$.
    \STATE \hspace{2em} Store tuple $(s_{k,t}, a_{k,t}, r_{k,t+1}, s_{k,t+1})$ in $\M_k$.
	\STATE \hspace{2em} Update step counter: $t \leftarrow t + 1$.
	\STATE \hspace{1em} \textbf{end while}
    \STATE \hspace{1em} \textbf{for} each sample $t \in \M_{k}$\textbf{:}
    \STATE \hspace{2em} Compute advantage function:\\
   	\hspace{2em} $
			\widehat{A}_{k,t}
		=
			\sum_{j=t}^{|\M_{k}|}
				\gamma^{j-t}
				r_{k,j+1}
				-
				v_{\xi_{k,i}}(s_{k,t})
     	$
	\STATE \hspace{1em} \textbf{end for}
    \STATE \hspace{1em} Compute distributed critic gradient:
		\begin{IEEEeqnarray*}{rCl}
			\widehat{\xi}_{k,i+1}
		&
		=
		&
			\xi_{k,i}
			+
			\frac{\alpha_{i+1}}{|\M_{k}|}
			\sum_{t=0}^{|\M_{k}|}
				\nabla_\xi v_{\xi_{k,i}} \left( s_{k,t} \right)
				\widehat{A}_{k,t}
		\\
			\xi_{k,i+1}
		&
		=
		&
			\sum_{l \in \N_k}
				c_{lk}
				\widehat{\xi}_{l,i+1}
		\end{IEEEeqnarray*}
		%
		%
    \STATE \hspace{1em} Compute distributed actor update:
		\begin{IEEEeqnarray*}{rCl}
			\widehat{w}_{k,i+1}
		&
		=
		&
			w_{k,i}
\\
		&&
			+
		\:
			\frac{\beta_{i+1}}{|\M_{k}|}
			\sum_{t=0}^{|\M_{k}|}
				\nabla_{w} \log \pi_{w_{k,i}}(a_{k,t}|s_{k,t})
			 	\widehat{A}_{k,t}
\qquad
		\\
			w_{k,i+1}
		&
		=
		&
			\sum_{l \in \N_k}
				c_{lk}
				\widehat{w}_{l,i+1}
		\end{IEEEeqnarray*}
	\STATE \hspace{1em} Update episode counter: $i \leftarrow i +1$.
    \STATE \textbf{end while}
\end{algorithmic}
\textbf{Return:} Critic and actor weights: $\xi_{k,E}$, $w_{k,E}$.
\end{algorithm}

%
\section{Numerical Experiments}
%

We evaluate the performance of Diff-DAC on three MRL problems of varying levels of difficulty. We use $\gamma=0.99$ for all tasks:

\textbf{Cart-pole balance:} 
We use the OpenAI Gym \cite{OpenAI2016gym} implementation,
but with continuous force. The action follows a Gaussian distribution with mean in the interval $[-10, 10]$.
The episode finishes when the pole is beyond $12$ degrees from vertical, 
cart moves more than $2.4$ units from the center,
or run for $200$ time-steps.
The single task uses parameters $(0.1, 0.5, 1.0)$
for the pole mass, pole half-length and cart mass, respectively.
The MRL problem consists of $25$ tasks:
pole mass in $\{0.1, 0.325, 0.55, 0.775, 1 \}$,
pole length in $\{ 0.05, 0.1625, 0.275, 0.3875, 0.5 \}$,
and cart mass $1$. 

\textbf{Inverted pendulum:}
The pendulum consists of a rigid pole and an actuated joint, with maximum torque clipped to interval $[-2, 2]$. 
The pendulum starts at a random angle in $[-\pi, \pi]$, with uniform distribution.
The goal is to take the pendulum to the upright position and balance.
The MRL problem consists of $25$ tasks with mass in $\{ 0.8, 0.9, 1.0, 1.1, 1.2 \}$,
and length in $\{ 0.8, 0.9, 1.0, 1.1, 1.2 \}$. 
The single task pole mass and length are $(1.0, 1.0)$.

\textbf{Cart-pole swing-up:} 
We extend cart-pole balance to the case where the pole starts from the bottom and the task is to swing up the pole to the upright position and balance. 
The reward function is $r = \frac{2}{1 + e^{d}} + \cos(\psi)$, 
where $d$ is the Euclidean distance of the pole from the track center and upright position, 
and $\psi$ is the pole angle.
This is a much more difficult task than standard cart-pole and more difficult than the inverted pendulum due to more complex dynamics.
The cart-pole swing-up MRL problem consists of $25$ tasks,
where pole mass is in $\{0.1, 0.2, 0.3, 0.4, 0.5 \}$,
pole half-length is in $\{ 0.2, 0.4, 0.6, 0.8, 1.0 \}$,
and cart mass is $0.5$. The single task uses parameters $(0.5, 0.25, 0.5)$
for the pole mass, pole half-length and cart mass respectively.

We compare Diff-DAC with Dist-MTLPS for the MRL problem in the cart-pole balance environment.
In particular, we compare against two variants of Dist-MTLPS, which consist in using two standard policy search methods, namely Reinforce \cite{williams1992simple} and PoWER \cite{kober2009policy}, 
for solving the individual tasks.
We only compare Diff-DAC with Dist-MTLPS in the cart-pole balance task, 
since the other two environments are uncontrollable with linear policies from raw data.
Our goal is to compare the performance of the single but expressive neural network policy provided by Diff-DAC with the task-specialised but less expressive linear policies provided by Dist-MTLPS. 

We also compare Diff-DAC with Cent-AC, which has only one agent
(central coordinator) that gathers and process samples from all the tasks synchronously and has the same neural network architectures and hyperparameters as Diff-DAC.
We remark that we use two versions of exactly the same vanilla actor-critic algorithm, where their only difference is whether there is a single agent with access to all the data (Cent-AC) or multiple networked agents with access to local datasets (Diff-DAC). 

Although experimenting with more environments and benchmarking with more algorithms would be helpful, 
testing in these simple environments is already useful, since they illustrate the behavior of the algorithms, without having to handle complex neural network architectures.

The network consists of $N=25$ agents, 
randomly deployed in a 2D world, 
with average degree $|\overline{\N}| \defeq \sum_{k=1}^N |\N_k| \approx 4.2$
(connectivity is determined by the distance between agents).
For Figure \ref{fig:connectivity}, we also include two additional networks of $N=25$ and $|\overline{\N}| =7.4$,
and $N=100$ and $|\overline{\N}| = 20$.
Matrix $C$ was obtained with the Hastigs-rule \cite[p.492]{sayed2014adaptation} in all cases, 
so that \eqref{eq:non-negative-coefficients}--\eqref{eq:aperiodic} hold.
We remark that the network topologies used in the experiments are not related to any form of task similarity, but they just reflect the sparse connectivity that appears naturally when agents and data are geographically distributed.

The critic and actor neural networks consist of $2$ hidden layers of $400$ neurons each with ReLu activation functions.
The output layer for the critic network is linear.
The output of the actor network includes a \textit{tanh} activation function that determines the mean of a normal distribution, and a \textit{Softplus} activation function that determines the variance for the normal distribution. 
We also included an extra penalty in the loss function equal to the entropy of the policy, 
with penalty coefficient $0.0005$.
Thus, both the mean and the variance are learned for the policy.
We use ADAM optimizer \cite{kingma2015adam},
with learning rate $0.01$ for critic and $0.001$ for actor. 
Diff-DAC performs a learning step ($i\leftarrow i+1$) every fifth episode.

The return of the tasks is reported as the (udiscounted) total rewards every 20 episodes and is averaged over $10$ test episodes at each point.
Figures show the median and first and third quartiles of the distribution of the average return of all the tasks.
Each epoch consists of $5$ episodes per epoch and per agent in Diff-DAC, 
and $5N$ episodes in total per epoch for Dist-MTLPS and Cent-AC, 
so that the three algorithms simulate the same number of episodes.
Every experiment was repeated at least $6$ times.

In Figure \ref{fig:cart-pole-balance} (bottom),
we observe that Diff-DAC learns faster than Dist-MTLPS Reinforce and reaches better asymptotic performance.
Dist-MTLPS PoWER converges faster than Diff-DAC, however the asymptotic performance of the latter is much better.
This is remarkable since Dist-MTLPS learns one different policy per task, 
while Diff-DAC learns a single policy common to all tasks.

We also observe that Diff-DAC converges slower than the Cent-AC, which was expected since the latter can compute the gradients with data from all tasks at every iteration, 
while the former has to wait until the parameters are diffused across the network.
However, Diff-DAC usually achieves higher asymptotic performance and less variance, in both single task and multitask problems,
which was also expected due to the already reported enhanced robustness against local optima of diffusion strategies for nonconvex optimization  problems \cite[Ch. 4]{valcarcel2017phdthesis}. This is shown in Figure \ref{fig:cart-pole-balance} (top), where Cent-AC tends to reach the optimal faster, but is unstable.
We guess that this may be alleviated by adding a \textit{replay buffer} \cite{mnih2013playing,Lillicrap2015Continuous}
or randomizing agents' samples to reduce their correlation, simulating asynchronocity (similar to asynchronous methods like A3C \cite{mnih2016asynchronous}).
However, our goal with these experiments is not to compare with SOTA centralized algorithms---which use several advancements to stabilise or improve performance---, but to evaluate whether diffusion is not only a feasible distributed strategy but also a valid alternative to stabilize learning.
Thus, we decided to use vanilla central actor-critic vs. vanilla distributed actor-critic, where their only difference is having a single agent with all the data vs. having multiple agents with local datasets.

Finally, Figure \ref{fig:connectivity} shows a simple experiment that studies the influence of the network topology. 
We evaluate Diff-DAC for the single-task cart-pole balance problem and see that for the same network size, $N=25$, 
a relatively sparse network, $|\overline{\N}| \approx N/6$, achieves performance similar to a more dense network, $|\overline{\N}| \approx N/3$.
In addition, we see that larger number of agents $N=100$ improves the asymptotic performance.

\begin{figure}[!h]
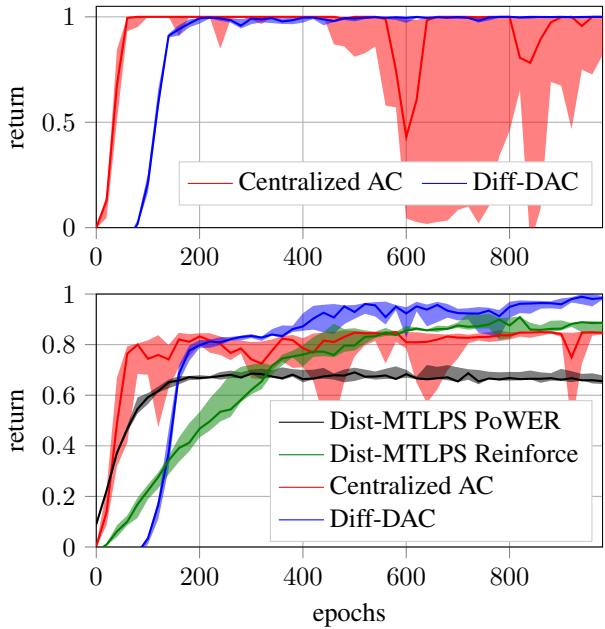

\input{cartpole_singletask_25_4_connectivity.tikz}\\
\input{cartpole_multitask_25_4_connectivity.tikz}
\vspace*{-10pt}
\caption{Cart-pole balance with continuous action for single-task (top) and multitask (bottom).
Cent-AC is faster than distributed approaches, but Diff-DAC achieves the best asymptotic performance.}
\label{fig:cart-pole-balance}
\end{figure}
\begin{figure}[!h]
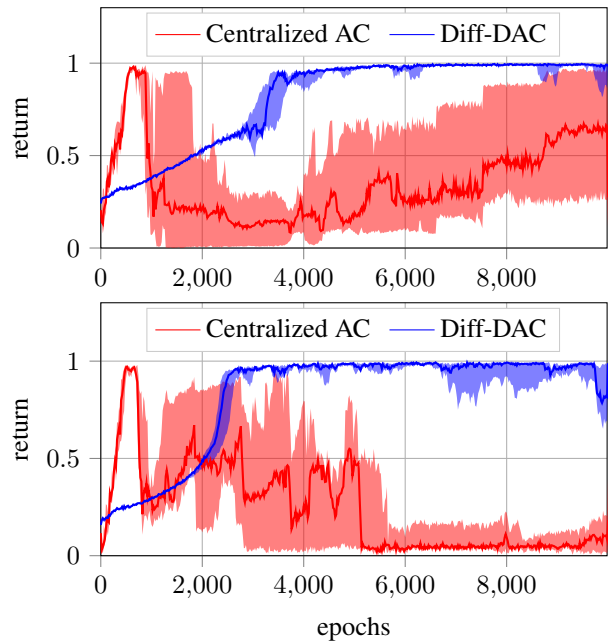

\input{pendulum_singletask_25_4_connectivity.tikz}
\input{pendulum_multitask_25_4_connectivity.tikz}
\vspace*{-10pt}
\caption{Inverted pendulum for single-task (top) and multitask (bottom).
Diff-DAC learns in both the single-task and multitask robustly. 
The central method learns the task quickly but is unstable.}
\end{figure}
\begin{figure}[!h]
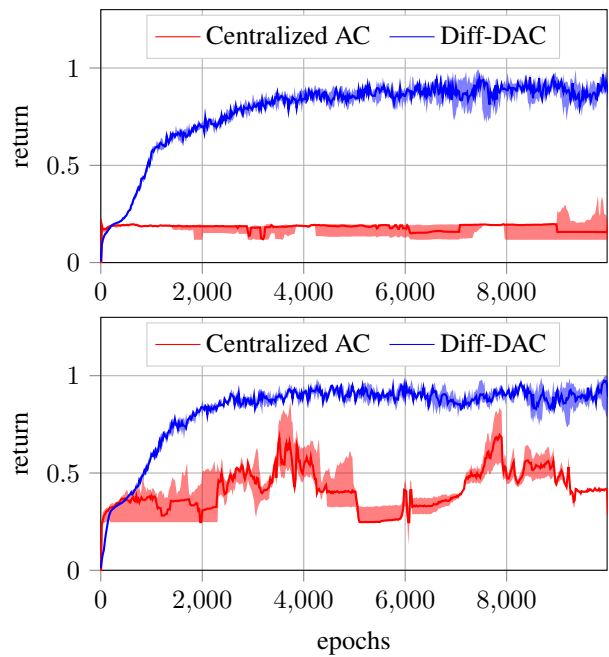

\input{cartpoleswingup_singletask_25_4_connectivity.tikz}\\
\input{cartpoleswingup_multitask_25_4_connectivity.tikz}
\vspace*{-10pt}
\caption{Cart-pole swing-up for single-task (top) and multitask (bottom).
Diff-DAC learns to swing-up and balance the pole consistently, while Cent-AC achieves much inferior performance.}
\label{fig:inverted-pendulum}
\end{figure}
\begin{figure}[!h]
\input{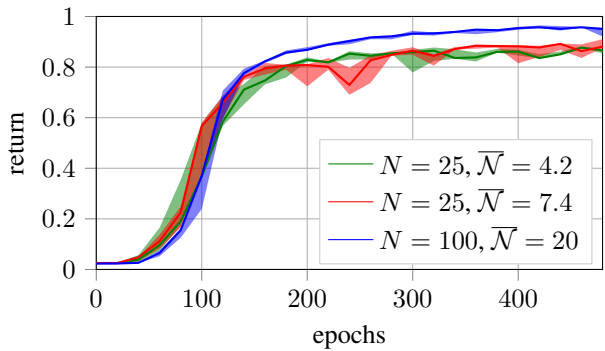}
\vspace*{-10pt}
\caption{Influence of network topology in single-task cart-pole balance with continuous action. Diff-DAC combines the experience of all agents, relatively insensitive to network sparsity.}
\label{fig:connectivity}
\end{figure}

%
\section{Conclusions}
\label{sec:conclusions}
%

We considered MRL where tasks are parametrized MDPs with parameters drawn from some distribution, 
and we derived an algorithm that learns a policy that performs well on average for the observed set of tasks.
We defined average global variables that allowed us to use standard optimal control theory and reformulate our MRL problem as an LP.
From this LP, we derived an exact, model-based actor-critic algorithm as an instance of dual ascent for finding the saddle point of the Lagrangian.
This saddle-point derivation is interesting in itself and provides novel insights in the actor-critic framework.
By approximating the exact method with deep neural networks, we obtained the Diff-DAC algorithm, which can scale to large number of tasks.

Simulation results showed that Diff-DAC can be faster and achieve higher asymptotic performance than the state of the art distributed algorithm for solving the MRL problem (i.e., Dist-MTLPS).
This is a remarkable result since the Diff-DAC agents converge to a single common policy that behaves better than the task-dependent linear policies obtained by Dist-MTLPS.
Moreover, Diff-DAC can solve complex problems that are uncontrollable from raw data by linear policies, 
while Dist-MTLPS requires (usually costly) feature engineering.
Diff-DAC is also very stable and achieves similar or usually higher asymptotic performance than the centraliized approach in both single and multitask problems.
This suggests that the sparse connectivity among agents induces a regularization effect that helps them to achieve better local optimum.
We consider this form of regularization an interesting line of rsearch.

Diff-DAC can be also extended to \textit{zero-shot learning} by taking the task parameter $\theta_k$ as an additional input to each agent's value/policy networks,
so that when a new task appears, it can input its parameter in the neural networks and adapt its behvarior to this task without further training.

Finally, it would be interesting to apply the proposed framework to derive distributed variants of other algorithms like PPO \cite{schulman2017proximal}.

\section{Acknowledgements}
%
We thank Haitham Bou-Ammar and Peter Vrancx for insightful discussions.

\bibliographystyle{aaai}
\bibliography{myreferences,myarticles}

\end{document}